\documentclass[letterpaper]{article}
\usepackage{amsmath,amssymb,graphicx}
\usepackage{aaai18}  
\usepackage{times}  
\usepackage{helvet}  
\usepackage{courier}  
\usepackage{url}  
\usepackage{graphicx}  
\begin{document}
%

\newtheorem{theorem}{Theorem}
\newtheorem{lemma}{Lemma}
\newtheorem{claim}{Claim}
\newtheorem{proposition}{Proposition}
\newtheorem{definition}{Definition}
\newtheorem{corollary}{Corollary}
\renewcommand{\phi}{\varphi}
\renewcommand{\epsilon}{\varepsilon}
\newcommand{\<}{\langle}
\renewcommand{\>}{\rangle}
\newenvironment{proof}{\noindent{\sf Proof.}}{\hfill $\boxtimes\hspace{2mm}$\linebreak}
\newcommand{\qed}{\hfill $\boxtimes\hspace{1mm}$}

\renewcommand{\H}{{\sf H}}
\newcommand{\K}{{\sf K}}
\renewcommand{\S}{{\sf S}}

\newsavebox{\diamonddotsavebox}
\sbox{\diamonddotsavebox}{$\Diamond$\hspace{-1.8mm}\raisebox{0.3mm}{$\cdot$}\hspace{1mm}}
\newcommand{\diamonddot}{\usebox{\diamonddotsavebox}}

\title{Time-Bounded Coalition Power in Nondeterministic Transition Systems}

\title{Time-Constrained Coalition Power in \\ Nondeterministic Transition Systems}

\title{Time-Constrained Coalition Power}

\title{A Little Agent Who Thought She Could}

\title{Strategies and Knowledge in Social Transition Systems}

\title{Coalition Power in Epistemic Transition Systems}

\title{Strategic Coalitions with Perfect Recall}

\author{Pavel Naumov \\Department of Computer Science\\Vassar College\\Poughkeepsie, New York 12604\\pnaumov@vassar.edu
\And  Jia Tao \\Department of Computer Science\\Lafayette College\\Easton, Pennsylvania 18042\\taoj@lafayette.edu}




%



\maketitle

\begin{abstract}
The paper proposes a bimodal logic that describes an interplay between distributed knowledge modality and coalition know-how modality. Unlike other similar systems, the one proposed here assumes perfect recall by all agents. Perfect recall is captured in the system by a single axiom. The main technical results are the soundness and the completeness theorems for the proposed logical system.
\end{abstract}

\section{Introduction}

Autonomous agents such as self-driving cars and robotic vacuum cleaners are facing the challenge of navigating without having complete information about the current situation. Such a setting could be formally captured by an {\em epistemic transition system} where an agent uses instructions to transition the system between states without being able to distinguish some of these states. In this paper we study properties of strategies in such systems. An example of such a system is the epistemic transition system $T_1$, depicted in Figure~\ref{intro-example-1 figure}. It has six states named $w_0,w'_0,w_1,w'_1,w_2,w'_2$ and two instructions $0$ and $1$ that an agent $a$ can use to transition the system from one state to another. For instance, if an instruction $0$ is given in state $w_0$, then the system transitions into state $w_1$. The system is called {\em epistemic} because the agent cannot distinguish state $w_i$ from state $w'_i$ for each $i=0,1,2$. The indistinguishability relation is shown in the figure using dashed lines. Atomic proposition $p$ is true only in state $w_2$.

\begin{figure}[ht]
\begin{center}
\scalebox{0.6}{\includegraphics{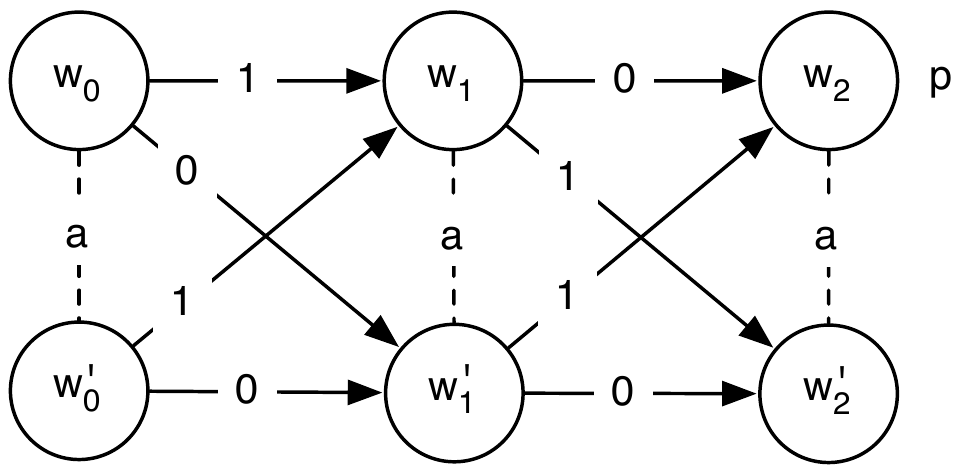}}
\caption{Epistemic Transition System $T_1$.}\label{intro-example-1 figure}
\end{center}
\end{figure}

The logical system that we propose consists of two modalities. The first is the knowledge modality $\K$. Imagine that the system starts in state $w_2$. Since agent $a$ cannot distinguish state $w_2$ from state $w'_2$ where statement $p$ is not satisfied, the agent does not know if $p$ is true or not. We write this as $(w_2)\nVdash \K_a p$. Next, suppose that the system started in state $w_1$ and the agent used instruction $0$ to transition the system into state $w_2$. In this paper we assume that all agents have perfect recall, so in state $w_2$ the agent remembers {\em history} $(w_1,0,w_2)$. However, such a history is indistinguishable from history $(w_1',0,w_2')$ because the agent cannot distinguish state $w_1$ from state $w'_1$ and state $w_2$ from state $w'_2$. Thus, the agent does not know that proposition $p$ is true in the state $w_2$ even with history $(w_1,0,w_2)$. 
We denote this by $(w_1,0,w_2)\nVdash \K_a p$. Finally, assume that the system started in state $w_0$ and the agent first used instruction $1$ to transition it into state $w_1$ and later instruction $0$ to transition it to state $w_2$. Thus, the history of the system is $(w_0,1,w_1,0,w_2)$. The only history that the agent cannot distinguish from this one is history $(w'_0,1,w_1,0,w_2)$. Since both of these histories end in a state where proposition $p$ is satisfied, agent $a$ {\em does} know that proposition $p$ is true in state $w_2$, given history $(w_0,1,w_1,0,w_2)$. We write this as $(w_0,1,w_1,0,w_2)\Vdash \K_a p$.

The other modality that we consider is the strategic power. In system $T_1$, the agent can transition the system from state $w_1$ to state $w_2$ by using instruction 0. Similarly, the agent can transition the system from state $w'_1$ to state $w_2$ by using instruction 1. In other words, given either history $(w_1)$ or history $(w'_1)$ the agent can transition the system to a state in which atomic proposition $p$ is satisfied.
We say that, given either history, agent $a$ has a {\em strategy} to achieve $p$. Histories  $(w_1)$ and $(w'_1)$ are the only histories indistinguishable by agent $a$ from history $(w_1)$. Since she has a strategy to achieve $p$ under all histories indistinguishable from history $(w_1)$, we say that given history $(w_1)$ the agent {\em knows that she has a strategy}. Similarly, given history $(w_1')$, she also knows that she has a strategy.
However, since indistinguishable histories $(w_1)$ and $(w_1')$ require different strategies to achieve $p$, given history $(w_1)$ she {\em does not know what the strategy is}. We say that she does not have a know-how strategy.  We denote this by $(w_1)\nVdash \H_a p$, where $\H$ stands for know-$\H$ow. Of course, it is also true that  $(w'_1)\nVdash \H_a p$.

The situation changes if the transition system starts in state $w_0$ instead of state $w_1$ and transitions to state $w_1$ under instruction $1$. Now the history is $(w_0,1,w_1)$ and the histories that the agent cannot distinguish from this one are history $(w'_0,1,w_1)$ and history $(w_0,1,w_1)$ itself. Given both of these two histories, agent $a$ can achieve $p$ using the same transition $0$. Thus, $(w_0,1,w_1)\Vdash \H_a p$. 

Finally note that there are only two histories: $(w_0)$ and $(w_0')$ indistinguishable from $(w_0)$. Given either history, agent $a$ can achieve $\H_a p$ using instruction $1$. Thus, $(w_0)\Vdash \H_a\H_a p$. That is, given history $(w_0)$ agent $a$ knows how to transition to a state in which formula $\H_a p$ is satisfied.

\subsubsection*{Multiagent Systems}
Like many other autonomous agents, self-driving cars are expected to use vehicle-to-vehicle communication to share traffic information and to coordinate actions~\cite{nhsa14}. Thus, it is natural to consider epistemic transition systems that have more than one agent. An example of such a system $T_2$ is depicted in Figure~\ref{intro-example-2 figure}. 
\begin{figure}[ht]
\begin{center}
\scalebox{0.6}{\includegraphics{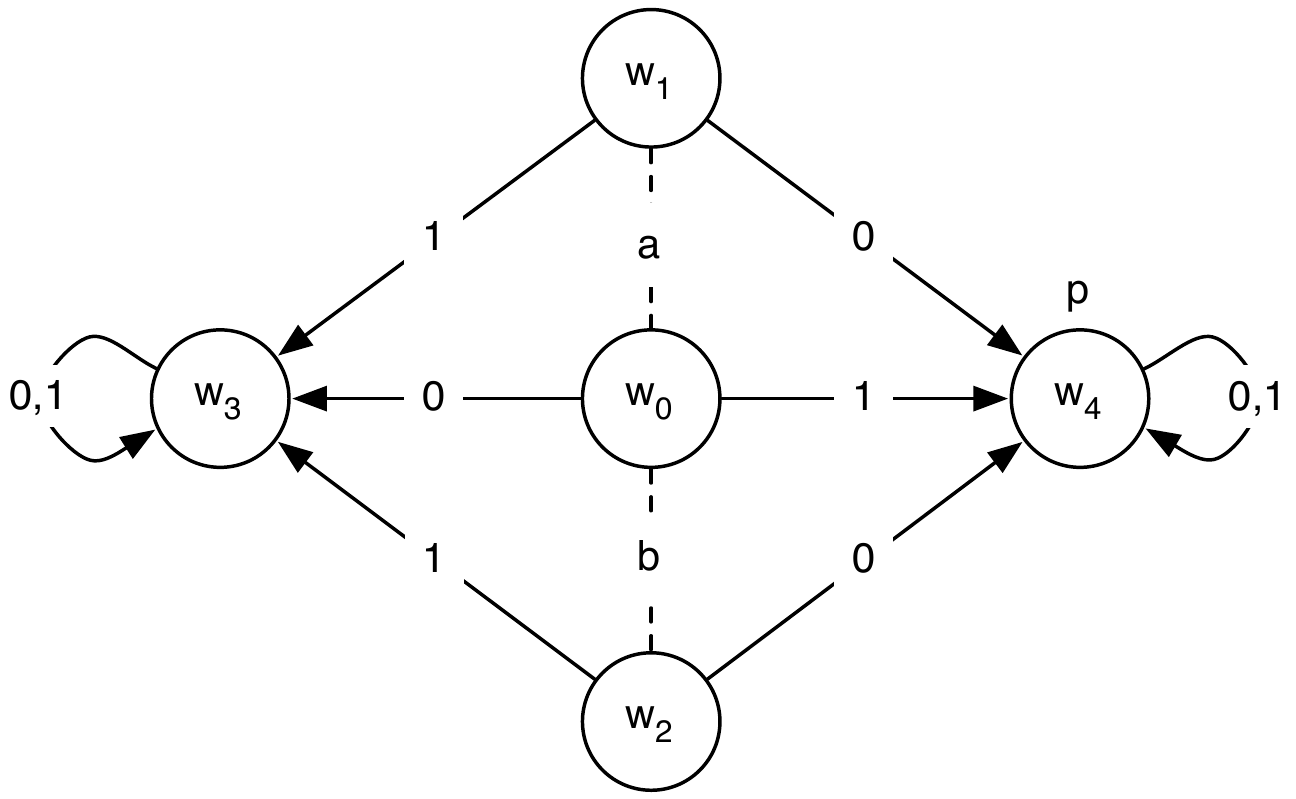}}
\caption{Epistemic Transition System $T_2$.}\label{intro-example-2 figure}
\end{center}
\vspace{0mm}
\end{figure}
This system has five epistemic states: $w_0,w_1,w_2,w_3$, and $w_4$ and three agents: $a$, $b$, and $c$. In each state the agents vote either 0 or 1 and the system transitions into the next state based on the majority vote. For example, since the directed edge from state $w_0$ to state $w_4$ is labelled with 1, if the majority of agents in state $w_0$ votes 1, then the system transitions into state $w_4$. Since coalition $\{a,b\}$ forms a majority, this coalition has a strategy to transition the system from state $w_0$ to state $w_4$ and, thus, to achieve $p$. Note that agent $a$ cannot distinguish state $w_0$ from state $w_1$ and thus agent $a$ does not know what she should vote for to achieve $p$. Similarly, agent $b$ also does not know what she should vote for to achieve $p$ because she cannot distinguish state $w_0$ from state $w_2$. In this paper, we assume that members of a coalition make the decisions based on combined (distributed) knowledge of the whole coalition.
In our example, coalition $\{a,b\}$ can distinguish state $w_0$ from both state $w_1$ and state $w_2$. Thus, given history $(w_0)$ the coalition $\{a,b\}$ knows how to achieve $p$. We denote this by $(w_0)\Vdash \H_{\{a,b\}} p$, or simply as $(w_0)\Vdash \H_{a,b} p$.

\subsubsection*{Universal Principles}

We have discussed a statement being true or false given a certain history. This paper focuses on the logical principles that are true for each history in each epistemic transition system. An example of such a principle is the {\em strategic positive introspection}:
$
\H_C\phi\to\K_C\H_C\phi.
$
This principle says that if a coalition knows how to achieve $\phi$, then the coalition knows that it knows how to achieve $\phi$. Informally, this principle is true because in order for statement $\H_C\phi$ to be satisfied for a given history $h$, coalition $C$ must have a strategy to achieve $\phi$ that works under any history $h'$ indistinguishable from history $h$ by the coalition. Thus, the same strategy must work for any history $h''$ indistinguishable from history $h'$ by the coalition. In other words, it is also true that $h'\Vdash\H_C\phi$. Recall that $h'$ is an arbitrary history indistinguishable from history $h$ by coalition $C$. Hence, $h\Vdash\K_C\H_C\phi$ according to the standard semantics of the epistemic modality $\K_C$. A similar argument can be used to justify the {\em strategic negative introspection}: 
$
\neg\H_C\phi\to\K_C\neg\H_C\phi.
$

Another universal principle is the {\em empty coalition} principle:
$
\K_\varnothing\phi\to\H_\varnothing\phi.
$
Indeed, $\K_\varnothing\phi$ means that statement $\phi$ is true under any history indistinguishable from the given history by an empty coalition. Since an empty coalition cannot distinguish any two histories, the assumption $\K_\varnothing\phi$ means that statement $\phi$ is true under any history. In particular, this statement is true after the next transition no matter how agents vote. Hence, $\H_\varnothing\phi$. 

The epistemic modality $\K_C$ also satisfies axioms of epistemic logic S5 for distributed knowledge. Know-how modality satisfies the {\em unachievability of falsehood} principle: $\neg\H_C\bot$, stating that no coalition can achieve $\bot$. Know-how modality also satisfies 
a form of 
{\em cooperation} principle~\cite{p01illc,p02}:
$$
\H_C(\phi\to\psi)\to(\H_D\phi\to\H_{C\cup D}\psi), \mbox{ where $C\cap D=\varnothing$}.
$$

\subsubsection*{Perfect Recall}

A complete trimodal logical system describing the interplay between distributed knowledge modality $\K_C$, coalition know-how modality $\H_C$, and standard (not know-how) strategic power modality in the {\em imperfect} recall setting was proposed by \cite{nt17tark}. We provide a complete axiomatization of the interplay between modalities $\K_C$ and $\H_C$ in the {\em perfect} recall setting. Surprisingly, the assumption of perfect recall by all agents is captured by a single principle that we call the {\em perfect recall} principle:
$
\H_D\phi\to\H_D\K_C\phi,
$
where $D\subseteq C\neq\varnothing$. This principle says that if a sub-coalition $D\subseteq C$ can achieve $\phi$, then after the vote the whole coalition will know that $\phi$ is true. Informally, this principle is true because coalition $C$ is able to recall how sub-coalition $D$ voted and, thus, will deduce that formula $\phi$ is true after the transition. 
As an empty coalition has no memory even in the perfect recall setting, it is essential for coalition $C$ to be nonempty. However, the sub-coalition $D$ can be empty.

\subsubsection*{Literature Review}  

Non-epistemic logics of coalition power were developed by \cite{p01illc,p02}, who also proved the completeness of the basic logic of coalition power. 
His approach has been widely studied in the literature~\cite{g01tark,vw05ai,b07ijcai,sgvw06aamas,abvs10jal,avw09ai,b14sr,gjt13jaamas}. An alternative logical system for coalition power was proposed by \cite{mn12tocl}. 

\cite{ahk02} introduced Alternating-Time Temporal Logic (ATL) that combines temporal and coalition modalities.
\cite{vw03sl} proposed to combine ATL with epistemic modality to form Alternating-Time Temporal Epistemic Logic. They did not prove the completeness theorem for the proposed logical system. Aminof, Murano, Rubin and Zuleger~\cite{amrz16kr} studied model-checking problems of an extension of ATL with epistemic and ``prompt eventually" modal operators.

\cite{aa12aamas} proposed a complete logical system that combines the coalition power and epistemic modalities. Since their system does not have epistemic requirements on strategies, it does not contain any axioms describing the interplay of these modalities. 
In the extended version, \cite{aa16jlc} added a complete axiomatization of an interplay between single-agent knowledge and know-how modalities.

Know-how strategies were studied before under different names. While \cite{ja07jancl} talked about ``knowledge to identify and execute a strategy", \cite{jv04fm} discussed ``difference between an agent knowing that he has a suitable strategy and knowing the strategy itself". \cite{v01ber} called such strategies ``uniform".  \cite{b08deon} investigated a related notion of ``knowingly doing", while \cite{bht09jancl} studied modality ``know they can do". \cite{w15lori,w17synthese} captured the ``knowing how" as a binary modality in a complete logical system with a single agent and without the knowledge modality.

Coalition know-how strategies for enforcing a condition indefinitely were investigated by \cite{nt17aamas}. 
Such strategies are similar to \cite[p. 80]{p01illc} ``goal maintenance" strategies in ``extended coalition logic".  A similar complete logical system in a {\em single-agent} setting for know-how strategies to {achieve} a goal in multiple steps rather than to {maintain} a goal is developed by \cite{fhlw17ijcai}.   

\cite{nt17tark} also proposed a complete trimodal logical system describing an interplay between distributed knowledge modality $\K_C$, coalition know-how modality $\H_C$, and standard (not know-how) strategic power modality in the {\em imperfect} recall setting. 

In this paper we provide a complete axiomatization of an interplay between modalities $\K_C$ and $\H_C$ in the {\em perfect} recall setting. The main challenge in proving the completeness, compared to~\cite{aa16jlc,fhlw17ijcai,nt17tark,nt17aamas}, is the need to construct not only ``possible worlds", but the entire ``possible histories", see the proof of Lemma~\ref{K child}.

\subsubsection*{Outline}


The rest of the paper is organized as follows. 
First, we introduce the syntax and semantics of our logical system. Next, we list the axioms and give examples of proofs in the system. Then, we prove the soundness and the completeness of this system. 

\section{Syntax and Semantics}\label{syntax and semantics}

Throughout the rest of the paper we assume a fixed set of agents $\mathcal A$. By $X^Y$ we denote the set of all functions from set $Y$ to set $X$, or in other words, the set of all tuples of elements from set $X$ indexed by the elements of set $Y$. If $t\in X^Y$ is such a tuple and $y\in Y$, then by $(t)_y$ we denote the $y$-th component of tuple $t$.

We now proceed to describe the formal syntax and semantics of our logical system starting with the definition of a transition system. Although our introductory examples used voting to decide on the next state of the system, in this paper we investigate universal properties of an arbitrary nondeterministic action aggregation mechanism. 

\begin{definition}\label{transition system}
A tuple $(W,\{\sim_a\}_{a\in \mathcal{A}},V,M,\pi)$ is called an epistemic transition system, if
\begin{enumerate}
    \item $W$ is a set of epistemic states,
    \item $\sim_a$ is an indistinguishability equivalence relation on $W$ for each $a\in\mathcal{A}$,
    \item $V$ is a nonempty set called domain of choices, 
    \item $M\subseteq W\times V^\mathcal{A}\times W$ is an aggregation mechanism,
    \item $\pi$ is a function that maps propositional variables into subsets of $W$.
\end{enumerate}
\end{definition}

For example, in the transition system $T_1$ depicted in Figure~\ref{intro-example-1 figure}, the set of states $W$ is $\{w_0,w_1,w_2,w_0',w_1',w_2'\}$ and relation $\sim_a$ is a transitive reflexive closure of $\{(w_0,w_0')$, $(w_1,w_1'),(w_2,w_2')\}$. 

Informally, an epistemic transition system is {\em regular} if there is at least one next state for each outcome of the vote.
\begin{definition}\label{regular}
An epistemic transition system $(W,\{\sim_a\}_{a\in \mathcal{A}},$ $V,M,\pi)$ is {\em regular} if for each $w\in W$ and each $\mathbf{s}\in V^\mathcal{A}$, there is $w'\in W$  such that $(w,\mathbf{s},w')\in M$.
\end{definition}

A {\em coalition} is a subset of $\mathcal{A}$.
A {\em strategy profile} of coalition $C$ is any tuple in the set $V^C$.

\begin{definition}\label{sim set}
For any states $w_1,w_2\in W$ and any coalition $C$, let $w_1\sim_C w_2$ if $w_1\sim_a w_2$ for each agent $a\in C$.
\end{definition}

\begin{lemma}\label{sim set lemma}
For each coalition $C$, relation $\sim_C$ is an equivalence relation on the set of epistemic states $W$.  \qed
\end{lemma}

\begin{definition}\label{s eq set}
For all strategy profiles $\mathbf{s}_1$ and $\mathbf{s}_2$ of coalitions $C_1$ and $C_2$ respectively and any coalition $C\subseteq C_1\cap C_2$, let $\mathbf{s}_1=_C\mathbf{s}_2$ if $(\mathbf{s}_1)_a=(\mathbf{s}_2)_a$ for each $a\in C$.
\end{definition}

\begin{lemma}\label{= eq rel lemma}
For any coalition $C$, relation $=_C$ is an equivalence relation on the set of all strategy profiles of coalitions containing coalition $C$. \qed
\end{lemma}

\begin{definition}\label{history}
A history is an arbitrary sequence $h=(w_0,\mathbf{s}_1,w_1,\mathbf{s}_2,w_2,\dots, \mathbf{s_n},w_n)$ such that $n\ge 0$ and
\begin{enumerate}
    \item $w_i\in W$ for each $i\le n$,
    \item $\mathbf{s_i}\in V^\mathcal{A}$ for each $i\le n$,
    \item $(w_i,\mathbf{s}_{i+1},w_{i+1})\in M$ for each $i< n$.
\end{enumerate}
\end{definition}

In this paper we assume that votes of all agents are private. Thus, an individual agent only knows her own votes and the equivalence classes of the states that the system has been at. This is formally captured in the following definition of indistinguishability of histories by an agent.
\begin{definition}\label{approx histories}
For any history $h=(w_0,\mathbf{s}_1,w_1,\dots, \mathbf{s_n},w_n)$, any history $h'=(w'_0,\mathbf{s'_1},w'_1,\dots, \mathbf{s'_{m}},w'_{m})$, and any agent $a\in \mathcal{A}$, let $h\approx_a h'$ if $n=m$ and
\begin{enumerate}
    \item $w_i\sim_a w'_i$ for each $i\le n$,
    \item $(\mathbf{s_i})_a=(\mathbf{s'_i})_a$ for each $i\le n$.
\end{enumerate}
\end{definition}

\begin{definition}\label{approx set}
For any histories $h_1,h_2$ and any coalition $C$, let $h_1\approx_C h_2$ if $h_1\approx_a h_2$ for each agent $a\in C$.
\end{definition}

\begin{lemma}\label{approx eq rel lemma}
For any coalition $C$, relation $\approx_C$ is an equivalence relation on the set of histories. \qed
\end{lemma}
The length $|h|$ of a history $h=(w_0,\mathbf{s}_1,w_1,\dots, \mathbf{s_n},w_n)$ is the value of $n$. By Definition~\ref{approx set}, the empty coalition cannot distinguish any two histories, even of different lengths.  
\begin{lemma}\label{history length lemma}
$|h_1|=|h_2|$ for each histories $h_1$ and $h_2$ such that $h_1\approx_C h_2$ for some nonempty coalition $C$. \qed
\end{lemma}
For any sequence $x=x_1,\dots,x_n$ and an element $y$, by sequence $x::y$ we mean $x_1,\dots,x_n,y$. If sequence $x$ is nonempty, then by $hd(x)$ we mean element $x_n$.
\begin{lemma}\label{history approx rec lemma}
If $(h_1::\mathbf{s}_1::w_1)\approx_C (h_2::\mathbf{s}_2::w_2)$, then $h_1\approx_C h_2$, $\mathbf{s}_1=_C\mathbf{s}_2$, and $w_1\sim_C w_2$. \qed
\end{lemma}

\begin{definition}\label{Phi}
Let $\Phi$ be the language specified as follows
$
\phi := p\;|\;\neg\phi\;|\;\phi\to\phi\;|\;\K_C\phi\;|\;\H_C\phi,
$ where $C\subseteq \mathcal{A}$.
\end{definition}

Boolean constants $\bot$ and $\top$ are defined as usual.

\begin{definition}\label{sat}
For any history $h$ of an epistemic transition system $(W,\{\sim_a\}_{a\in \mathcal{A}},V,M,\pi)$ and any formula $\phi\in \Phi$, let satisfiability relation $h\Vdash\phi$ be defined as follows
\begin{enumerate}
    \item $h\Vdash p$ if $hd(h)\in \pi(p)$ and $p$ is a propositional variable,
    \item $h\Vdash\neg\phi$ if $h\nVdash\phi$,
    \item $h\Vdash\phi\to\psi$ if $h\nVdash\phi$ or $h\Vdash\psi$,
    \item $h\Vdash \K_C\phi$ if $h'\Vdash\phi$ for each history $h'$ s.t. $h\approx_C h'$,
    \item $h\Vdash\H_C\phi$ if there is a strategy profile $\mathbf s\in V^C$ such that for any history $h'::\mathbf{s}'::w'$, if $h\approx_C h'$ and $\mathbf{s}=_C\mathbf{s}'$, then $h'::\mathbf{s}'::w'\Vdash\phi$.
\end{enumerate}
\end{definition}

\section{Axioms}\label{axioms section}

In additional to propositional tautologies in language $\Phi$, our logical system consists of the following axioms: 
\begin{enumerate}
    \item Truth: $\K_C\phi\to\phi$,
    \item Negative Introspection: $\neg\K_C\phi\to\K_C\neg\K_C\phi$,
    \item Distributivity: $\K_C(\phi\to\psi)\to(\K_C\phi\to\K_{C}\psi)$,
    \item Monotonicity: $\K_C\phi\to \K_D\phi$, if $C\subseteq D$,
    \item Strategic Positive Introspection: $\H_C\phi\to\K_C\H_C\phi$,
    \item Cooperation: $\H_C(\phi\to\psi)\to(\H_D\phi\to\H_{C\cup D}\psi)$, where $C\cap D=\varnothing$,
    \item Empty Coalition: $\K_\varnothing\phi\to\H_\varnothing \phi$,
    \item Perfect Recall: $\H_D\phi\to\H_D\K_C\phi$, where $D\subseteq C\neq\varnothing$,
    \item Unachievability of Falsehood: $\neg\H_C\bot$.
\end{enumerate}

We write $\vdash \phi$ if formula $\phi$ is provable from the axioms of our logical system using Necessitation, Strategic Necessitation, and Modus Ponens inference rules:
$$
\dfrac{\phi}{\K_C\phi}
\hspace{10mm}
\dfrac{\phi}{\H_C\phi}
\hspace{10mm}
\dfrac{\phi,\hspace{5mm} \phi\to\psi}{\psi}.
$$
We write $X\vdash\phi$ if formula $\phi$ is provable from the theorems of our logical system and a set of additional axioms $X$ using only Modus Ponens inference rule. 

The next lemma follows from a well-known observation that the Positive Introspection axiom is provable from the other axioms of S5. 

\begin{lemma}\label{positive introspection lemma}
$\vdash \K_C\phi\to\K_C\K_C\phi$.\qed
\end{lemma}
\begin{proof}
Formula $\neg\K_C\phi\to\K_C\neg\K_C\phi$ is an instance of the Negative Introspection axiom. Thus, $\vdash \neg\K_C\neg\K_C\phi\to \K_C\phi$ by the law of contrapositive in the propositional logic. Hence,
$\vdash \K_C(\neg\K_C\neg\K_C\phi\to \K_C\phi)$ by the Necessitation inference rule. Thus, by  the Distributivity axiom and the Modus Ponens inference rule, 
\begin{equation}\label{pos intro eq 2}
   \vdash \K_C\neg\K_C\neg\K_C\phi\to \K_C\K_C\phi.
\end{equation}

At the same time, $\K_C\neg\K_C\phi\to\neg\K_C\phi$ is an instance of the Truth axiom. Thus, $\vdash \K_C\phi\to\neg\K_C\neg\K_C\phi$ by contraposition. Hence, taking into account the following instance of  the Negative Introspection axiom $\neg\K_C\neg\K_C\phi\to\K_C\neg\K_C\neg\K_C\phi$,
one can conclude that $\vdash \K_C\phi\to\K_C\neg\K_C\neg\K_C\phi$. The latter, together with statement~(\ref{pos intro eq 2}), implies the statement of the lemma by the laws of propositional reasoning.
\end{proof}








\section{Derivation Examples}\label{examples section}

This section contains examples of formal proofs in our logical system. The results obtained here are used in the proof of completeness. The proof of Lemma~\ref{strategic negative introspection lemma} is based on the one proposed to us by Natasha Ale\-chi\-na.


\begin{lemma}[Alechina]\label{strategic negative introspection lemma}
$\vdash \neg\H_C\phi\to\K_C\neg\H_C\phi$.
\end{lemma}
\begin{proof}
By the Positive Strategic Introspection axiom, 
$\vdash \H_C\phi\to\K_C\H_C\phi$. Thus, $\vdash \neg\K_C\H_C\phi\to \neg\H_C\phi$ by the contrapositive. Hence, $\vdash \K_C(\neg\K_C\H_C\phi\to \neg\H_C\phi)$ by the Necessitation inference rule. Then, by the Distributivity axiom and the Modus Ponens inference rule
$\vdash \K_C\neg\K_C\H_C\phi\to \K_C\neg\H_C\phi$. 
Thus, by the Negative Introspection axiom and the laws of propositional reasoning,
$\vdash \neg\K_C\H_C\phi\to \K_C\neg\H_C\phi$. Note that $\neg\H_C\phi\to\neg\K_C\H_C\phi$ is the contrapositive of the Truth axiom. Therefore, by the laws of propositional reasoning,
$\vdash \neg\H_C\phi\to \K_C\neg\H_C\phi$.
\end{proof}



\begin{lemma}\label{subset lemma H}
$\vdash\H_C\phi\to \H_D\phi$, where $C\subseteq D$.
\end{lemma}
\begin{proof}
Note that $\phi\to\phi$ is a propositional tautology. Thus, $\vdash\phi\to\phi$. Hence, $\vdash\H_{D\setminus C}(\phi\to\phi)$ by the Strategic Necessitation inference rule. At the same time, by the Cooperation axiom,
$
\vdash\H_{D\setminus C}(\phi\to\phi)\to(\H_C\phi\to\H_D\phi)
$
due to the assumption $C\subseteq D$.  Therefore, $\vdash\H_C\phi\to\H_D\phi$ by the Modus Ponens inference rule.
\end{proof}

\begin{lemma}\label{superdistributivity lemma} If $\phi_1,\dots,\phi_n\vdash\psi$, then
\begin{enumerate}
    \item $\K_C\phi_1,\dots,\K_C\phi_n\vdash\K_C\psi$,
    \item $\H_{C_1}\phi_1,\dots,\H_{C_n}\phi_n\vdash\H_{\bigcup_{i=1}^nC_i}\psi$, where sets $C_1,\dots,C_n$ are pairwise disjoint.
\end{enumerate}
\end{lemma}
\begin{proof}
To prove the second statement, apply deduction lemma for propositional logic $n$ time. Then, we have $\vdash\phi_1\to(\dots\to(\phi_n\to\psi))$. Thus, by the Strategic Necessitation inference rule, $\vdash\H_\varnothing(\phi_1\to(\dots\to(\phi_n\to\psi)))$ Hence, $\vdash\H_{C_1}\phi_1\to\H_{C_1}(\phi_2\dots\to(\phi_n\to\psi))$ by the Cooperation axiom and the Modus Ponens inference rule. Then, $\H_{C_1}\phi_1\vdash\H_{C_1}(\phi_2\dots\to(\phi_n\to\psi))$ by the Modus Ponens inference rule. Thus, again by the Cooperation axiom and the Modus Ponens inference rule we have $\H_{C_1}\phi_1\vdash\H_{C_2}\phi_2\to \H_{C_1\cup C_2}(\phi_3\dots\to(\phi_n\to\psi))$.
Therefore, by repeating the last two steps $n-2$ times, $\H_{C_1}\phi_1,\dots,\H_{C_n}\phi_n\vdash\H_{\bigcup_{i=1}^nC_i}\psi$. The proof of the first statement is similar, but it uses the Distributivity axiom instead of the Cooperation axiom. 
\end{proof}

\section{Soundness}

In this section we prove the following soundness theorem for our logical system.

\begin{theorem}\label{soundness theorem}
If $\vdash\phi$, then $h\Vdash\phi$ for each history of each regular epistemic transition system.
\end{theorem}

The proof of the soundness of axioms of epistemic logic S5 with distributed knowledge is standard. Below we prove the soundness of each remaining axiom as a separate lemma.

\begin{lemma}
If $h\Vdash \H_C\phi$, then $h\Vdash\K_C\H_C\phi$.
\end{lemma}
\begin{proof}
Due to Definition~\ref{sat}, assumption $h\Vdash \H_C\phi$ implies that there is a strategy profile $\mathbf{s}_0$ of coalition $C$ such that for any history $h'::\mathbf{s}'::w'$, if $h\approx_C h'$ and $\mathbf{s}_0=_C\mathbf{s}'$, then $h'::\mathbf{s}'::w'\Vdash\phi$.

Consider any history $h''$ such that $h\approx_C h''$. By Definition~\ref{sat}, it suffices to show that $h''\Vdash \H_C\phi$. Furthermore, by the same Definition~\ref{sat}, it suffices to prove that for any history $h'''::\mathbf{s}'''::w'''$, if $h''\approx_C h'''$ and $\mathbf{s}_0=_C\mathbf{s}'''$, then $h'''::\mathbf{s}'''::w'''\Vdash\phi$.

Suppose that $h''\approx_C h'''$ and $\mathbf{s}_0=_C\mathbf{s}'''$.
By Lemma~\ref{approx eq rel lemma}, statements $h\approx_C h''$ and $h''\approx_C h'''$ imply that $h\approx_C h'''$. Therefore, $h'''::\mathbf{s}'''::w'''\Vdash\phi$ by the choice of $\mathbf{s}_0$.
\end{proof}
\begin{lemma}
If $h\Vdash\H_C(\phi\to\psi)$, $h\Vdash\H_D\phi$, and $C\cap D=\varnothing$, then $h\Vdash\H_C\psi$.
\end{lemma}
\begin{proof}
By Definition~\ref{sat}, assumption $h\Vdash \H_C(\phi\to\psi)$ implies that there is a strategy profile $\mathbf{s}_1$ of coalition $C$ such that for any history $h'::\mathbf{s}'::w'$, if $h\approx_C h'$ and $\mathbf{s}_1=_C\mathbf{s}'$, then $h'::\mathbf{s}'::w'\Vdash\phi\to\psi$.

Similarly, assumption $h\Vdash \H_D\phi$ implies that there is a profile $\mathbf{s}_2$ of coalition $D$ such that for any history $h'::\mathbf{s}'::w'$, if $h\approx_D h'$ and $\mathbf{s}_2=_D\mathbf{s}'$, then $h'::\mathbf{s}'::w'\Vdash\phi$.

Consider strategy profile $\mathbf{s}$ of coalition $C\cup D$ such that
$$
(\mathbf{s})_a=
\begin{cases}
(\mathbf{s}_1)_a, & \mbox{ if $a\in C$},\\
(\mathbf{s}_2)_a, & \mbox{ if $a\in D$}.
\end{cases}
$$
Strategy profile $\mathbf{s}$ is well-defined because coalitions $C$ and $D$ are disjoint. By Definition~\ref{sat}, it suffices to show that for any history $h'::\mathbf{s}'::w'$, if $h\approx_{C\cup D} h'$ and $\mathbf{s}=_{C\cup D}\mathbf{s}'$, then $h'::\mathbf{s}'::w'\Vdash\psi$. 

Suppose that  $h\approx_{C\cup D} h'$ and $\mathbf{s}=_{C\cup D}\mathbf{s}'$. Then, $h\approx_C h'$ and $\mathbf{s}_1=_{C}\mathbf{s}=_{C\cup D}\mathbf{s}'$. Hence, $h'::\mathbf{s}'::w'\Vdash\phi\to\psi$ by the choice of strategy $\mathbf{s}_1$. Similarly, $h'::\mathbf{s}'::w'\Vdash\phi$. Therefore, $h'::\mathbf{s}'::w'\Vdash\psi$ by Definition~\ref{sat}.
\end{proof}
\begin{lemma}
If $h\Vdash\K_\varnothing\phi$, then $h\Vdash\H_\varnothing\phi$.
\end{lemma}
\begin{proof}
Let $\mathbf{s}\in V^\varnothing$ be the empty tuple.  By Definition~\ref{sat}, it suffices to show that for any history $h'::\mathbf{s}'::w'$ we have $h'::\mathbf{s}'::w'\Vdash\phi$.  Definition~\ref{approx set} implies $h\approx_\varnothing (h'::\mathbf{s}'::w')$.  Thus, $h'::\mathbf{s}'::w'\Vdash\phi$
 by Definition~\ref{sat} and $h\Vdash\K_\varnothing\phi$. 
\end{proof}
\begin{lemma}
If $h\Vdash\H_D\phi$, then $h\Vdash\H_D\K_C\phi$, where $D\subseteq C$ and $C\neq \varnothing$.
\end{lemma}
\begin{proof}
By Definition~\ref{sat}, assumption $h\Vdash \H_D\phi$ implies that there is a strategy profile $\mathbf{s}$ of coalition $D$ such that for any history $h'::\mathbf{s}'::w'$, if $h\approx_D h'$ and $\mathbf{s}=_D\mathbf{s}'$, then $h'::\mathbf{s}'::w'\Vdash\phi$.

Consider any history $h_1::\mathbf{s}_1::w_1$ such that $h\approx_D h_1$ and $\mathbf{s}=_D\mathbf{s}_1$. By Definition~\ref{sat}, it suffices to prove that $h_1::\mathbf{s}_1::w_1\Vdash\K_C\phi$. Let $h_2$ be any such history that $(h_1::\mathbf{s}_1::w_1)\approx_C h_2$. Again by Definition~\ref{sat}, it suffices to prove that $h_2\Vdash\phi$. 

By Lemma~\ref{history length lemma}, assumptions $(h_1::\mathbf{s}_1::w_1)\approx_C h_2$ and $C\neq \varnothing$ imply that $|h_2|=|h_1::\mathbf{s}_1::w_1|\ge 1$. Thus, $h_2=h'_2::\mathbf{s}'_2::w'_2$ for some history $h'_2$, some complete strategy profile $\mathbf{s}'_2$, and some epistemic state $w'_2$. 

Then, $(h'_2::\mathbf{s}'_2::w'_2)=h_2\approx_C (h_1::\mathbf{s}_1::w_1)$. Hence, 
$h_1\approx_C h'_2$ and $\mathbf{s}_1=_C \mathbf{s}'_2$ by Lemma~\ref{history approx rec lemma}. Hence, $h\approx_D h_1\approx_C h'_2$ and $\mathbf{s}=_D\mathbf{s}_1=_C\mathbf{s}'_2$ by the choice of history $h_1::\mathbf{s}_1::w_1$. Then, $h\approx_D h'_2$ and $\mathbf{s}=_D\mathbf{s}'_2$ by Lemma~\ref{approx eq rel lemma}, Lemma~\ref{= eq rel lemma}, and because $D\subseteq C$. Thus, $h'_2::\mathbf{s}'_2::w'_2\Vdash\phi$ by the choice of strategy profile $\mathbf{s}$. Therefore, $h_2\Vdash\phi$, because $(h'_2::\mathbf{s}'_2::w'_2)=h_2$. 
\end{proof}
%
%
%
\begin{lemma}
$h\nVdash\H_C\bot$ for any history $h$ of any regular epistemic transition system.
\end{lemma}
\begin{proof}
Suppose $h\Vdash\H_C\bot$. By Definition~\ref{sat}, there is a strategy profile $\mathbf{s}\in V^C$ such that for any history $h'::\mathbf{s}'::w'$, if $h\approx_C h'$ and $\mathbf{s}=_C\mathbf{s}'$, then  $h'::\mathbf{s}'::w'\Vdash\bot$.  

By Definition~\ref{transition system}, set $V$ contains at least one element $v_0$. Let $\mathbf{s}'$ be a complete strategy profile such that
\begin{equation}\label{choice of s'}
(\mathbf{s}')_a=
\begin{cases}
(\mathbf{s})_a,& \mbox{ if $a\in C$},\\
v_0, & \mbox{ otherwise}.
\end{cases}
\end{equation}
By Definition~\ref{regular}, there is an epistemic state $w'\in S$ such that $(hd(h),\mathbf{s}',w')\in M$. Thus, $h::\mathbf{s}'::w'$ is a history by Definition~\ref{history}. Note that $h\approx_C h$ by Lemma~\ref{approx eq rel lemma} and $\mathbf{s}=_C\mathbf{s}'$ due to equation (\ref{choice of s'}). Thus, $h::\mathbf{s}'::w'\Vdash\bot$ by the choice of strategy profile $\mathbf{s}$, which  contradicts Definition~\ref{sat} and the definition of $\bot$.
\end{proof}
This concludes the proof of Theorem~\ref{soundness theorem}.

\section{Completeness}

In the rest of this paper we focus on the completeness theorem for our logical system with respect to the class of regular epistemic transition systems. We start the proof of completeness by fixing a maximal consistent set of formulae $X_0$ and defining a canonical epistemic transition system $ETS(X_0)=(W,\{\sim_a\}_{a\in \mathcal{A}},\Phi,M,\pi)$ using the ``unravelling" technique~\cite{s75slfm}. Note that the domain of choices in the canonical model is the set of all formulae $\Phi$.

\subsubsection*{Canonical Epistemic Transition System}

\begin{definition}\label{canonical worlds}
The set of epistemic states $W$ consists of all sequences $X_0,C_1,X_1,\dots,C_n,X_n$, such that $n\ge 0$ and
\begin{enumerate}
    \item $X_i$ is a maximal consistent subset of $\Phi$ for each $i\ge 1$,
    \item $C_i\subseteq \mathcal{A}$ for each $i\ge 1$,
    \item  $\{\phi\;|\;\K_{C_i}\phi\in X_{i-1}\}\subseteq X_i$, for each $i\ge 1$.
\end{enumerate}
\end{definition}

\begin{definition}\label{canonical sim}
Suppose that $w=X_0,C_1,X_1,\dots,C_n,X_n$ and $w'=X_0,C'_1,X'_1 \dots,C'_m,X'_m$ are epistemic states. For any agent $a\in\mathcal{A}$, let $w\sim_a w'$ if there is a non-negative integer $k\le \min\{n,m\}$ such that
\begin{enumerate}
    \item $X_i=X'_i$ for each $i$ such that $0<i\le k$,
    \item $C_i=C'_i$ for each $i$ such that $0<i\le k$,
    \item $a\in C_i$ for each $i$ such that $k<i\le n$,
    \item $a\in C'_i$ for each $i$ such that $k<i\le m$.
\end{enumerate}
\end{definition}

\begin{lemma}\label{down lemma}
For any epistemic state $X_0,C_1,\dots,C_n,X_n$ and any integer $k\le n$, if $\K_C\phi\in X_n$ and $C\subseteq C_i$ for each integer $i$ such that $k<i\le n$, then $\K_C\phi\in X_k$.
\end{lemma}
\begin{proof}
Suppose $\K_C\phi\notin X_k$ for some $k\le n$. Let $m$ be the maximal such $k$. Note that $m<n$ by the assumption $\K_C\phi\in X_n$ of the lemma. Thus, $m< m+1\le n$.

Assumption $\K_C\phi\notin X_{m}$ implies $\neg\K_C\phi\in X_m$ by the maximality of the set $X_{m}$. Hence, $X_{m}\vdash \K_C\neg\K_C\phi$ by the Negative Introspection axiom. Thus, $X_{m}\vdash \K_{C_{m+1}}\neg\K_C\phi$ by the Monotonicity axiom and the assumption $C\subseteq C_{m+1}$ of the lemma (recall that $m+1\le n$). 
Then, $\K_{C_{m+1}}\neg\K_C\phi\in X_{m}$ due to the maximality of the set $X_{m}$.  Hence, $\neg\K_C\phi\in X_{m+1}$ by Definition~\ref{canonical worlds}.
Thus, $\K_C\phi\notin X_{m+1}$ due to the consistency of the set $X_{m+1}$, which contradicts the choice of $m$.
\end{proof}

\begin{lemma}\label{up lemma}
For any epistemic state $X_0,C_1,\dots,C_n,X_n$ and any integer $k\le n$, if $\K_C\phi\in X_{k}$ and $C\subseteq C_i$ for each integer $i$ such that $k<i\le n$, then $\phi\in X_n$.
\end{lemma}
\begin{proof}
We prove the lemma by induction on the distance between $n$ and $k$. In the base case $n=k$. Then the assumption $\K_C\phi\in X_{n}$ implies $X_n\vdash\phi$ by the Truth axiom. Therefore, $\phi\in X_n$ due to the maximality of set $X_n$.

Suppose that $k<n$. Assumption $\K_C\phi\in X_{k}$ implies $X_k\vdash \K_C\K_C\phi$ by Lemma~\ref{positive introspection lemma}. Thus, $X_k\vdash \K_{C_{k+1}}\K_C\phi$ by the Monotonicity axiom, the condition $k<n$ of the inductive step, and the assumption $C\subseteq C_{k+1}$ of the lemma. Then, $\K_{C_{k+1}}\K_C\phi\in X_k$ by the maximality of set $X_k$. 
Hence, $\K_C\phi\in X_{k+1}$ by Definition~\ref{canonical worlds}. Therefore, $\phi\in X_n$ by the induction hypothesis. 
\end{proof}

\begin{lemma}\label{up-down lemma}
For any epistemic states $w,w'\in W$ such that $w\sim_C w'$, if $\K_C\phi\in hd(w)$, then $\phi\in hd(w')$.
\end{lemma}
\begin{proof}
The statement of the lemma follows from Lemma~\ref{down lemma} and  Lemma~\ref{up lemma} as well as Definition~\ref{canonical sim} because there is a unique path between any two nodes in a tree.
\end{proof}

Next, we specify the aggregation mechanism of the canonical epistemic transition system. Informally, if a coalition has a know-how strategy to achieve $\phi$ and all members of the coalition vote for $\phi$, then $\phi$ must be true after the transition.
\begin{definition}\label{canonical M}
For any states $w,w'\in W$ and any complete strategy profile $\mathbf{s}\in\Phi^\mathcal{A}$, let $(w,\mathbf{s},w')\in M$ if
\begin{eqnarray*}
&\{\phi\,|\,(\H_D\phi\in hd(w))  \wedge \forall a\in D((\mathbf{s})_a=\phi)\}\subseteq hd(w').
\end{eqnarray*}
\end{definition}

\begin{definition}\label{canonical pi}
$\pi(p)=\{w\in W\;|\; p\in hd(w)\}$.
\end{definition}

This concludes the definition of the canonical epistemic transition system $ETS(X_0)=(W,\{\sim_a\}_{a\in \mathcal{A}},\Phi,M,\pi)$. We prove that this system is regular in Lemma~\ref{canonical is regular}.

\subsubsection*{$\K$-child Lemmas}


The following technical results (Lemmas~\ref{state K child lemma}--\ref{K child}) about the knowledge modality $\K$ are used in the proof of completeness. 
\begin{lemma}\label{state K child lemma}
For any epistemic state $w\in W$ if $\neg\K_C\phi\in hd(w)$, then there is an epistemic state $w'\in W$ such that $w\sim_C w'$ and  $\neg\phi\in hd(w')$.
\end{lemma}
\begin{proof}
Consider the set
$X=\{\neg\phi\}\cup\{\psi\;|\;\K_C\psi\in hd(w)\}$.
First, we show that set $X$ is consistent. Assume the opposite.
Then, there exist formulae $\K_C\psi_1,\dots,\K_C\psi_n\in hd(w)$ such that
$
\psi_1,\dots,\psi_n\vdash\phi
$.
Thus, 
$\K_C\psi_1,\dots,\K_C\psi_n\vdash\K_C\phi$
by Lemma~\ref{superdistributivity lemma}.
Therefore, $hd(w)\vdash \K_C\phi$ by the choice of formulae $\K_C\psi_1,\dots,\K_C\psi_n$, which contradicts the consistency of the set $hd(w)$ due to the assumption $\neg\K_C\phi\in hd(w)$.

Let $\hat{X}$ be a maximal consistent extension of set $X$ and let $w'$ be sequence $w::C::X$. Note that $w'\in W$ by Definition~\ref{canonical worlds} and the choice of set $X$. Furthermore, $w\sim_C w'$ by Definition~\ref{canonical sim}. To finish the proof, note that $\neg\phi\in X\subseteq \hat{X}=hd(w')$ by the choice of set $X$. 
\end{proof}

History $h$ is a sequence whose last element $hd(h)$ is an epistemic state. Epistemic state $hd(h)$, by Definition~\ref{canonical worlds}, is also a sequence. Expression $hd(hd(h))$ denotes the last element of the sequence $hd(h)$.
\begin{lemma}\label{K child all}
For any history $h$, if $\K_C\phi\in hd(hd(h))$, then $\phi\in hd(hd(h'))$ for each history $h'$ such that $h\approx_C h'$.
\end{lemma}
\begin{proof}
Assumption $h\approx_C h'$ by Definition~\ref{approx histories} implies that $hd(h)\sim_C hd(h')$. Therefore, $\phi\in hd(hd(h'))$ by Lemma~\ref{up-down lemma} and the assumption $\K_C\phi\in hd(hd(h))$.
\end{proof}
\begin{figure}[ht]
\begin{center}
\scalebox{0.6}{\includegraphics{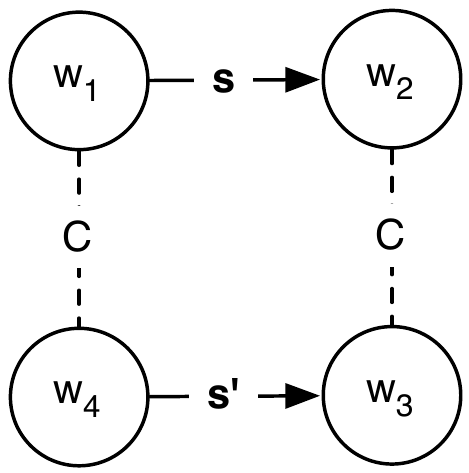}}
\caption{Illustration for Lemma~\ref{square lemma}.}\label{square figure}
\end{center}
\vspace{0mm}
\end{figure}
\begin{lemma}\label{square lemma}
For any nonempty coalition $C$, any states $w_1,w_2,w_3\in W$, and any complete strategy profile $\mathbf{s}$ such that $(w_1,\mathbf{s},w_2)\in M$ and $w_2\sim_C w_3$, see Figure~\ref{square figure}, there is a state $w_4$ and a complete strategy profile  $\mathbf{s}'$ such that $w_1\sim_C w_4$, $(w_4,\mathbf{s}', w_3)\in M$, and $\mathbf{s}=_C\mathbf{s}'$.
\end{lemma}
\begin{proof}
Let $\mathbf{s}'$ be a complete strategy profile such that
\begin{equation}
\label{s' def}
(\mathbf{s}')_a=
\begin{cases}
(\mathbf{s})_a, & \mbox{ if } a\in C,\\
\bot, & \mbox{ otherwise}.
\end{cases}
\end{equation}
Consider the set of formulae
\begin{eqnarray*}
X&=&\{\phi\;|\;\K_C\phi\in hd(w_1)\}\cup
\\
&&\{\neg\H_D\psi\;|\; \neg\psi\in hd(w_3) \wedge\forall a\in D((\mathbf{s}')_a=\psi)\}.
\end{eqnarray*}
First, we show that set $X$ is consistent. Indeed, set
$$
\{\neg\H_D\psi\;|\; \neg\psi\in hd(w_3) \wedge \forall a\in D((\mathbf{s}')_a=\psi)\}
$$ is equal to the union of the following two sets
\begin{eqnarray*}
&\hspace{-4mm}\{\neg\H_D\psi\;|\; \neg\psi\in hd(w_3)\wedge D\subseteq C \wedge \forall a\in D((\mathbf{s}')_a=\psi)\},\\
&\hspace{-4mm}\{\neg\H_D\psi\;|\; \neg\psi\in hd(w_3) \wedge D\nsubseteq C \wedge \forall a\in D((\mathbf{s}')_a=\psi)\}.
\end{eqnarray*}
The second set is a subset of $\{\neg\H_D\bot\,|\,D\subseteq\mathcal{A}\}$ by the choice of strategy $\mathbf{s}'$, see (\ref{s' def}).
Thus, set $X$ is a subset of 
\begin{eqnarray*}
&&\hspace{-7mm}\{\phi\;|\;\K_C\phi\in hd(w_1)\}\cup \{\neg\H_D\bot\,|\,D\subseteq\mathcal{A}\}\cup \\
&&\hspace{-7mm}\{\neg\H_D\psi\;|\; \neg\psi\in hd(w_3) \wedge D\subseteq C \wedge\forall a\in D((\mathbf{s}')_a=\psi) \}.
\end{eqnarray*}
Hence, by the Unachievability of Falsehood axiom, to show the consistency of set $X$ it suffices to prove the consistency of the union of the set $\{\phi\;|\;\K_C\phi\in hd(w_1)\}$
and the set
\begin{eqnarray*}
&&\hspace{-7mm}\{\neg\H_D\psi\;|\; \neg\psi\in hd(w_3) \wedge D\subseteq C \wedge\forall a\in D((\mathbf{s}')_a=\psi)\}.
\end{eqnarray*}
Suppose the opposite. In other words, assume that there are
\begin{eqnarray}
\hspace{-10mm}\mbox{formulae}&& \K_C\phi_1,\dots,\K_C\phi_n\in hd(w_1),\label{choice of phis}\\
\hspace{-10mm}\mbox{formulae}&&\neg\psi_1,\dots,\neg\psi_m\in hd(w_3),\label{choice of psis}\\
\hspace{-10mm}\mbox{and sets}&& D_1,\dots,D_m\subseteq C,\label{choice of Ds}\\
\hspace{-5mm}\mbox{such that}&&\forall i\le m\;\forall a\in D_i\;((\mathbf{s}')_a=\psi_i)\label{s' equation},\\
\hspace{-10mm}\mbox{and}&&\phi_1,\dots,\phi_n,\neg\H_{D_1}\psi_1,\dots,\neg\H_{D_m}\psi_m\vdash \bot.\nonumber
\end{eqnarray}
By Lemma~\ref{superdistributivity lemma} and the Truth axiom, 
$$
\K_C\phi_1,\dots,\K_C\phi_n,
\K_C\neg\H_{D_1}\psi_1,\dots,\K_C\neg\H_{D_m}\psi_m
\vdash \bot.
$$
Hence, statement~(\ref{choice of phis}) and the consistency of the set $hd(w_1)$ imply that  there exists $k\le m$ such that $\K_C\neg\H_{D_{k}}\psi_{k}\notin hd(w_1)$. Thus, $\neg\K_C\neg\H_{D_{k}}\psi_{k}\in hd(w_1)$ due to the maximality of the set $hd(w_1)$. 
Then, $\neg\K_{D_{k}}\neg\H_{D_{k}}\psi_{k}\in hd(w_1)$ by statement~(\ref{choice of Ds}) and the contrapositive of the Monotonicity axiom.
Then, $hd(w_1)\vdash \H_{D_{k}}\psi_{k}$ by the contrapositive of Lemma~\ref{strategic negative introspection lemma}. 
Thus, $hd(w_1)\vdash \H_{D_{k}}\K_C\psi_{k}$ by the Perfect Recall axiom, statement~(\ref{choice of Ds}), and the assumption of the lemma that $C\ne \varnothing$.  
Hence, $\H_{D_{k}}\K_C\psi_{k}\in hd(w_1)$
due to the maximality of the set $hd(w_1)$.
Note that statements~(\ref{s' def}) and (\ref{choice of Ds}) imply $\mathbf{s}=_{D_{k}}\mathbf{s}'$. Then, $(\mathbf{s})_a=\psi_{k}$ for each agent $a\in D_{k}$ by statement~(\ref{s' equation}). Thus, $\K_{C}\psi_{k}\in hd(w_2)$ by assumption $(w_1,\mathbf{s},w_2)\in M$, statement (\ref{choice of Ds}), and Definition~\ref{canonical M}.
Hence, $\psi_{k}\in hd(w_3)$ by the assumption $w_2\sim_C w_3$ of the lemma and Lemma~\ref{up-down lemma}. This contradicts statement~(\ref{choice of psis}) and the consistency of the set $hd(w_3)$. Therefore, set $X$ is consistent.

Let $\hat{X}$ be any maximal consistent extension of set $X$. Define $w_4$ to be the sequence $w_1::C::\hat{X}$. Note that $w_4\in W$ by Definition~\ref{canonical worlds} and the choice of set $X$. At the same time, $w_1\sim_C w_4$ by Definition~\ref{canonical sim} and Definition~\ref{sim set}. 

Finally, let us show that $(w_4,\mathbf{s}',w_3)\in M$ using Definition~\ref{canonical M}. Consider any $D\subseteq \mathcal{A}$ and any $\H_D\psi\in hd(w_4)=\hat{X}$ such that $\mathbf{s}'(a)=\psi$ for each $a\in D$. We need to show that $\psi\in hd(w_3)$. Suppose the opposite. Then, $\neg\psi\in hd(w_3)$ by the maximality of set $hd(w_3)$. Thus, $\neg\H_D\psi\in X$ by the choice of set $X$. Hence, $\neg\H_D\psi\in X\subseteq\hat{X}$. Therefore, $\H_D\psi\notin \hat{X}$ due to the consistency of set $\hat{X}$, which contradicts the choice of formula $\H_D\psi$.
\end{proof}
\begin{lemma}\label{K child empty}
For any history $h$, if $\K_\varnothing\phi\notin hd(hd(h))$, then there is a history $h'$ s.t. $h\approx_\varnothing h'$ and $\neg\phi\in hd(hd(h'))$.
\end{lemma}
\begin{proof}
By Lemma~\ref{state K child lemma}, there is a state $w\in W$ such that $hd(h)\sim_C w$ and $\neg\phi\in hd(w)$. Let $h'$ be a one-element sequence $w$. Note that $h\approx_\varnothing h'$ by Definition~\ref{approx set}. Finally, $\neg\phi\in hd(w)=hd(hd(h'))$.
\end{proof}
\begin{lemma}\label{K child}
For any nonempty coalition $C$ and any history $h$, if $\K_C\phi\notin hd(hd(h))$, then there is a history $h'$ such that $h\approx_C h'$ and $\neg\phi\in hd(hd(h'))$.
\end{lemma}
\begin{proof}
Let $h=(w_0,\mathbf{s}_1,w_1,\dots, \mathbf{s_n},w_n)$. We prove the lemma by induction on integer $n$. 

\noindent{\em Base Case.} Let $n=0$. By Lemma~\ref{state K child lemma}, there is $w'_0\in W$ such that $w_0\sim_C w'_0$ and $\neg\phi\in hd(w_0')$. Let $h'$ be a one-element sequence $w_0'$. Note that $w_0\sim_C w'_0$ implies that $h\approx_C h'$ by Definition~\ref{approx histories}. Also, $\neg\phi\in hd(w'_0)=hd(hd(h'))$.
\begin{figure}[ht]
\begin{center}
\vspace{0mm}
\scalebox{0.6}{\includegraphics{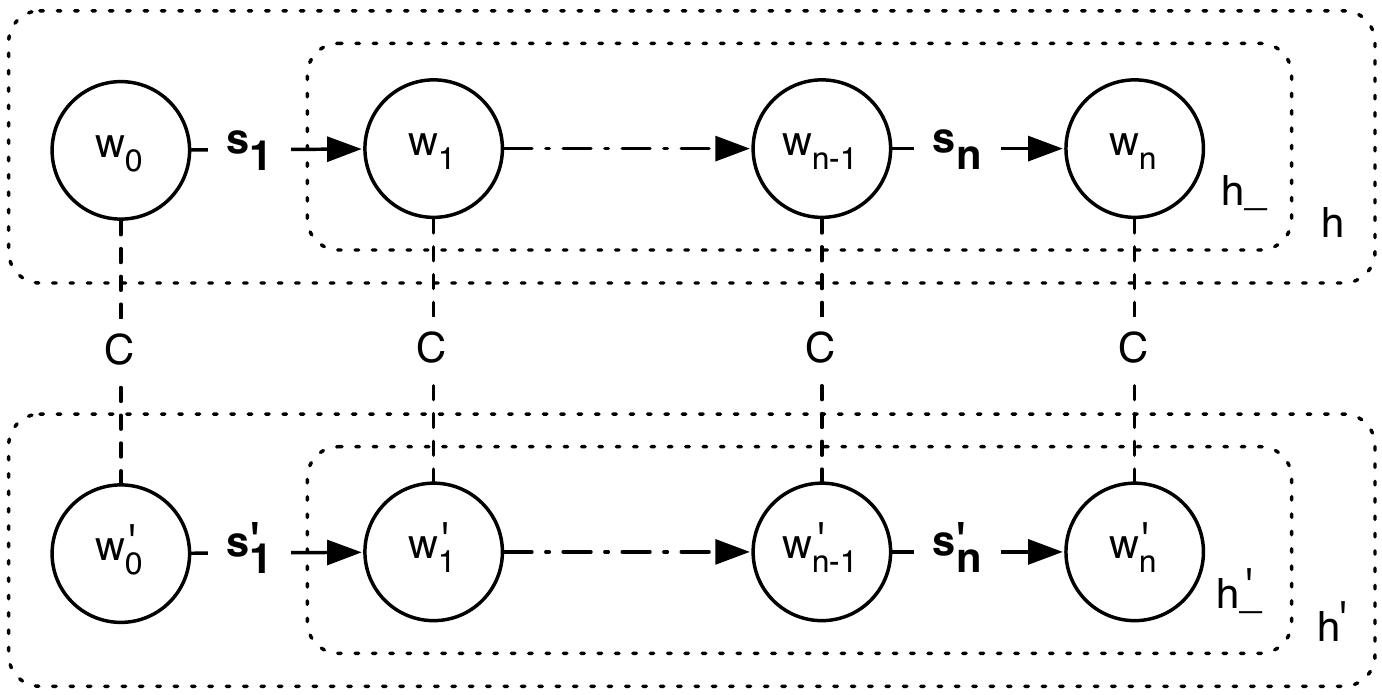}}
\caption{Histories $h$ and $h'$.}\label{making-history figure}
\vspace{0mm}
\end{center}
\end{figure}

\noindent{\em Induction Step.} Let $h_-=(w_1,\mathbf{s}_2,\dots, \mathbf{s_n},w_n)$, see Figure~\ref{making-history figure}. By Definition~\ref{history}, sequence $h_{-}$ is a history. By the induction hypothesis there is a history $h'_-$ such that $h_-\approx_C h'_-$ and $\neg\phi\in hd(hd(h'_-))$. Let $h'_-=(w'_1,\mathbf{s'_2},\dots, \mathbf{s'_n},w'_n)$. %
By Lemma~\ref{square lemma}, there is a state $w'_0$ and a complete strategy profile $\mathbf{s_1'}$ such that $w_0\sim_C w'_0$, $(w'_0,\mathbf{s_1'},w'_1)\in M$, and $\mathbf{s}_1=_C\mathbf{s_1'}$. 

Define $h'=(w'_0,\mathbf{s'_1},w'_1,\mathbf{s'_2},\dots, \mathbf{s'_n},w'_n)$. 
By Definition~\ref{history}, statement $(w'_0,\mathbf{s_1'},w'_1)\in M$ implies that $h'$ is a history. 
Note that $h\approx_C h'$ by Definition~\ref{approx histories} because $h'_-\approx_C h_-$, $w_0\sim_C w'_0$, and $\mathbf{s}_1=_C\mathbf{s_1'}$. 
Finally, $\neg\phi\in hd(hd(h'))$ because $hd(h')=hd(h'_-)$ and $\neg\phi\in hd(hd(h'_-))$.
\end{proof}

\subsubsection*{$\H$-child Lemmas}

Lemmas~\ref{state H child all lemma} and \ref{state H child lemma} are about the know-how modality $\H$. They are used later in the proof.

\begin{lemma}\label{state H child all lemma}
For any history $h$, if $\H_C\phi\in hd(hd(h))$, then there is a strategy profile $\mathbf{s}$ of coalition $C$ s.t.  for any history $h'::\mathbf{s}'::w'$ if $h\approx_C h'$ and $\mathbf{s}=_C\mathbf{s}'$, then $\phi\in hd(w')$.
\end{lemma}


\begin{proof}
Consider a strategy profile $\mathbf{s}$ of coalition $C$ such that $(\mathbf{s})_a=\phi$ for each agent $a\in C$. Suppose that $\H_C\phi\in hd(hd(h))$ and sequence $h'::\mathbf{s}'::w'$ is a history such that  $h\approx_C h'$ and $\mathbf{s}=_C\mathbf{s}'$. It suffices to show that $\phi\in hd(w')$.

By the Strategic Positive Introspection axiom, assumption $\H_C\phi\in hd(hd(h))$ implies $hd(hd(h))\vdash \K_C\H_C\phi$. Thus, $\K_C\H_C\phi \in hd(hd(h))$ by the maximality of set $hd(hd(h))$.

Assumption $h\approx_C h'$ implies $hd(h)\sim_C hd(h')$ by Definition~\ref{approx histories}. Thus, $\H_C\phi \in hd(hd(h'))$ by Lemma~\ref{up-down lemma} and because $\K_C\H_C\phi \in hd(hd(h))$. 


By Definition~\ref{history}, assumption that sequence $h'::\mathbf{s}'::w'$ is a history implies that $(hd(h'),\mathbf{s}',w')\in M$. Thus, $\phi\in hd(w')$ by Definition~\ref{canonical M} because $\H_C\phi \in hd(hd(h'))$ and  $(\mathbf{s}')_a=(\mathbf{s})_a=\phi$ for each agent $a\in C$. 
\end{proof}

\begin{lemma}\label{state H child lemma}
For any history $h$ and any strategy profile $\mathbf{s}$ of a coalition $C$, if $\neg\H_C\phi\in hd(hd(h))$, then there is a history $h::\mathbf{s}'::w'$ such that $\mathbf{s}=_C\mathbf{s}'$ and  $\neg\phi\in hd(w')$.
\end{lemma}
\begin{proof}
Let $w=hd(h)$. Consider a complete strategy profile $\mathbf{s}'$ and a set of formulae $X$ such that
\begin{equation}\label{first s' def}
(\mathbf{s}')_a=
\begin{cases}
(\mathbf{s})_a, & \mbox{ if } a\in C,\\
\top, & \mbox{ otherwise,}
\end{cases}
\end{equation}
\begin{eqnarray*}
\mbox{and }&&X=\{\neg\phi\}\cup\{\chi\;|\; \K_\varnothing\chi\in hd(w)\} \cup\\
&&\hspace{7mm}\{\psi\;|\;H_D\psi\in hd(w) \wedge \forall a\in D((\mathbf{s}')_a=\psi)\}.
\end{eqnarray*}
First, we show that set $X$ is consistent. Indeed, note that set
$
\{\psi\;|\; \H_D\psi\in hd(w) \wedge \forall a\in D((\mathbf{s}')_a=\psi)\}
$ is equal to the union of the following two sets
\begin{eqnarray*}
&&\{\psi\;|\; \H_D\psi\in hd(w)\wedge D\subseteq C \wedge \forall a\in D((\mathbf{s}')_a=\psi)\},\\
&&\{\psi\;|\; \H_D\psi\in hd(w) \wedge D\nsubseteq C \wedge \forall a\in D((\mathbf{s}')_a=\psi)\}.
\end{eqnarray*}
The second of the two sets is a subset of $\{\top\}$ by the choice (\ref{first s' def}) of strategy profile $\mathbf{s}'$.
Thus, set $X$ is a subset of
\begin{eqnarray*}
&& \{\top\}\cup\{\neg\phi\}\cup \{\chi\;|\; \K_\varnothing\chi\in hd(w)\} \cup\\
&&\{\psi\;|\; \H_D\psi\in hd(w) \wedge D\subseteq C \wedge \forall a\in D((\mathbf{s}')_a=\psi)\}.
\end{eqnarray*}
Hence, to show the consistency of set $X$, it suffices to prove the consistency of the union of 
$\{\neg\phi\}\cup \{\chi\;|\; \K_\varnothing\chi\in hd(w)\}$
and
$\{\psi\;|\; \H_D\psi\in hd(w) \wedge D\subseteq C \wedge \forall a\in D((\mathbf{s}')_a=\psi)\}$.
Suppose the opposite. In other words, assume there are
\begin{eqnarray}
\mbox{formulae} && \K_\varnothing\chi_1,\dots,\K_\varnothing\chi_n\in hd(w)\label{first choice of chis}\\
\hspace{-20mm}\mbox{and formulae} && \H_{D_1}\psi_1,\dots,\H_{D_m}\psi_m\in hd(w)\label{first choice of psis}\\
\mbox{such that} && D_1,\dots,D_m\subseteq C,\label{first choice of Ds}\\
\mbox{} &&\forall i\le m\;\forall a\in D_i\;((\mathbf{s}')_a=\psi_i),\label{first s' equation}
\end{eqnarray}
and $\chi_1,\dots,\chi_n,\psi_1,\dots,\psi_m\vdash \phi$. By Lemma~\ref{superdistributivity lemma},
$$
\H_\varnothing\chi_1,\dots,\H_\varnothing\chi_n,\H_{D_1}\psi_1,\dots,\H_{D_m}\psi_m\vdash\H_{\cup_{i=1}^m D_i}\phi.
$$
Then, by the Empty Coalition axiom, 
$$
\K_\varnothing\chi_1,\dots,\K_\varnothing\chi_n,\H_{D_1}\psi_1,\dots,\H_{D_m}\psi_m\vdash\H_{\cup_{i=1}^m D_i}\phi.
$$
Thus, by Lemma~\ref{subset lemma H} and assumption~(\ref{first choice of Ds}), 
$$
\K_\varnothing\chi_1,\dots,\K_\varnothing\chi_n,\H_{D_1}\psi_1,\dots,\H_{D_m}\psi_m\vdash\H_{C}\phi.
$$
Hence, $hd(w)\vdash\H_{C}\phi$ by assumption~(\ref{first choice of chis}) and assumption~(\ref{first choice of psis}). Then, $\neg\H_{C}\phi\notin hd(w)$ due to the consistency of the set $hd(w)$, which contradicts the assumption $\neg\H_{C}\phi\in hd(w)$ of the lemma. Therefore, set $X$ is consistent.

Let set $\hat{X}$ be a maximal consistent extension of set $X$ and $w'$ be the sequence $w::\varnothing::\hat{X}$. Note that $w'\in W$ by Definition~\ref{canonical worlds} and because
$\{\chi\;|\; \K_\varnothing\chi\in hd(w)\}\subseteq X\subseteq \hat{X}=hd(w')$ 
by the choice of set $X$. 

Also note that $(w,\mathbf{s}', w')\in M$ by Definition~\ref{canonical M} and because $\{\psi\;|\;\H_D\psi\in hd(w) \wedge \forall a\in D((\mathbf{s}')_a=\psi)\}
\subseteq X\subseteq \hat{X}=hd(w')$
by the choice of set $X$. Thus, $h::\mathbf{s}'::w'$ is a history by Definition~\ref{history}.

Finally, $\neg\phi\in hd(w')$ because $\neg\phi\in X\subseteq \hat{X}=hd(w')$ by the choice of set $X$.
\end{proof}
\begin{lemma}\label{canonical is regular}
The system $ETS(X_0)$ is regular. 
\end{lemma}
\begin{proof}
 Let $w\in W$ and $\mathbf{s}\in V^\mathcal{A}$. By Definition~\ref{regular}, it suffices to show that there is an epistemic state $w'$ such that $(w,\mathbf{s},w')\in M$. Indeed, let history $h$ be a single-element sequence $w$. Note that $\neg\H_\mathcal{A}\bot\in hd(w)=hd(hd(h))$ by the Unachievability of Falsehood axiom and due to the maximality of set $hd(w)$. Thus, by Lemma~\ref{state H child lemma}, there is a history $h::\mathbf{s}'::w'$ such that $\mathbf{s}=_\mathcal{A}\mathbf{s'}$. Hence, $(hd(h),\mathbf{s}',w')\in M$ by Definition~\ref{history}. 
 At the same time, $\mathbf{s}=_\mathcal{A}\mathbf{s'}$ implies that $\mathbf{s}=\mathbf{s'}$ by Definition~\ref{s eq set}. Thus, $(hd(h),\mathbf{s},w')\in M$. Therefore, $(w,\mathbf{s},w')\in M$ because $hd(h)=w$.
\end{proof}

\subsubsection*{Completeness: Final Steps}

\begin{lemma}\label{main induction}
$h\Vdash\phi$ iff $\phi\in hd(hd(h))$ for each history $h$ and each formula $\phi\in\Phi$.
\end{lemma}
\begin{proof}
We prove the statement by induction on the structural complexity of formula $\phi$. If $\phi$ is an atomic proposition $p$, then $h\Vdash p$ iff $hd(h)\in \pi(p)$, by Definition~\ref{sat}. Hence, $h\Vdash p$ iff $p\in hd(hd(h))$ by Definition~\ref{canonical pi}.

The cases when formula $\phi$ is a negation or an implication follow from Definition~\ref{sat} and the maximality and the consistency of the set $hd(hd(h))$ in the standard way.

Next, suppose that formula $\phi$ has the form $\K_C\psi$.

\noindent$(\Rightarrow):$ Let  $\K_C\psi\notin hd(hd(h))$. Then, either by Lemma~\ref{K child empty} (when set $C$ is empty) or by Lemma~\ref{K child} (when set $C$ is nonempty), there is a history $h'$ such that $h\approx_C h'$ and $\neg\psi\in hd(hd(h'))$. Hence, $h'\nVdash\psi$ by the induction hypothesis. 
Therefore, $h\nVdash\K_C\psi$ by Definition~\ref{sat}.

\noindent$(\Leftarrow):$ Let  $h\nVdash\K_C\psi$. By Definition~\ref{sat}, there is a history $h'$ such that $h\approx_C h'$ and $h'\nVdash\psi$. Thus, $\psi\notin hd(hd(h'))$ by the induction hypothesis. Then, $K_C\phi\notin hd(hd(h))$ by Lemma~\ref{K child all}. 

Finally, let formula $\phi$ be of the form $\H_C\psi$.

\noindent$(\Rightarrow):$ Assume $h\Vdash\H_C\psi$. Then, by Definition~\ref{sat}, there is a strategy profile $\mathbf s\in \Phi^C$ such that for any history $h'::\mathbf{s}'::w'$, if $h\approx_C h'$ and $\mathbf{s}=_C\mathbf{s}'$, then $h'::\mathbf{s}'::w'\Vdash\psi$. Thus, by Lemma~\ref{approx eq rel lemma}, 
\begin{equation}\label{our minipage}
    \begin{minipage}{0.35\textwidth}
    for any history $h::\mathbf{s}'::w'$, if  $\mathbf{s}=_C\mathbf{s}'$, then $h::\mathbf{s}'::w'\Vdash\psi$.
    \end{minipage}
\end{equation}

Suppose that $\H_C\psi\notin hd(hd(h))$. Then, $\neg\H_C\psi\in hd(hd(h))$ due to the maximality of the set $hd(hd(h))$. Hence,  by Lemma~\ref{state H child lemma}, there is a history $h::\mathbf{s}'::w'$ such that $\mathbf{s}=_C\mathbf{s}'$ and $\neg\psi\in hd(w')$. Thus, $\psi\notin hd(w')$ due to the consistency of set $hd(w')$. Hence, by the induction hypothesis, $h::\mathbf{s}'::w'\nVdash\psi$, which contradicts statement~(\ref{our minipage}). 

\noindent$(\Leftarrow):$ Assume that $\H_C\psi\in hd(hd(h))$. By Lemma~\ref{state H child all lemma}, there is a strategy profile $\mathbf{s}\in\Phi^C$ such that for any history $h'::\mathbf{s}'::w'$ if $h\approx_C h'$ and $\mathbf{s}=_C\mathbf{s}'$, then $\psi\in hd(w')$. Hence, by the induction hypothesis, for any history $h'::\mathbf{s}'::w'$ if $h\approx_C h'$ and $\mathbf{s}=_C\mathbf{s}'$, then $h'::\mathbf{s}'::w'\Vdash \psi$. Therefore, $h\Vdash\H_C\psi$ by Definition~\ref{sat}.
\end{proof}
\begin{theorem}\label{completeness theorem}
If $h\Vdash\phi$ for each history $h$ of each regular epistemic transition system, then $\vdash\phi$.
\end{theorem}
\begin{proof}
Suppose that $\nvdash\phi$. Consider any maximal consistent set $X_0$ such that $\neg\phi\in X_0$. Let $h_0$ be a single-element sequence consisting of just set $X_0$. By Definition~\ref{history}, sequence $h_0$ is a history of the canonical epistemic transition system $ETS(X_0)$. 
Then, $h_0\Vdash \neg\phi$ by Lemma~\ref{main induction}. Therefore, $h_0\nVdash \phi$ by Definition~\ref{sat}.
\end{proof}
Theorem~\ref{completeness theorem} can be generalized to the strong completeness theorem in the standard way. We also believe that the number of states and the domain of choices in the canonical model can be made finite using filtration on subformulae.

\section{Conclusion}

We have extended the recent study of the interplay between knowledge and strategic coalition power~\cite{aa16jlc,fhlw17ijcai,nt17tark,nt17aamas} to the case of perfect recall. Our main results are the soundness and the completeness theorems. 

\bibliographystyle{aaai}

\end{document}